\def\year{2017}\relax

\documentclass[letterpaper]{article}
\usepackage{aaai17}
\usepackage{times}
\usepackage{helvet}
\usepackage{courier}
\usepackage{url} 
\usepackage{graphicx} 

\usepackage{amsmath,amsthm,amssymb,amsfonts}
\usepackage{mdframed}
\usepackage{framed}
\usepackage{subfigure}
\usepackage[linesnumbered,ruled]{algorithm2e}
\usepackage{hhline}
\newcommand{\E}{E}
\newtheorem{theorem}{Theorem}

\newtheorem{definition}{Definition}
\newtheorem{observation}{Observation}
\DeclareMathOperator*{\argmin}{arg\,min}
\DeclareMathOperator*{\argmax}{arg\,max}
\newcommand{\la}{$L$-updates policy }

\begin{document}
\title{Minimizing Maximum Regret \\ in Commitment Constrained Sequential Decision Making}
\author{
Qi Zhang\\
University of Michigan\\
qizhg@umich.edu\\
\And
Satinder Singh\\
University of Michigan\\
baveja@umich.edu\\
\And
Edmund Durfee\\
University of Michigan\\
durfee@umich.edu\\
}
\maketitle
\begin{abstract}
In cooperative multiagent planning, it can often be beneficial for an agent to make commitments about aspects of its behavior to others, allowing them in turn to plan their own behaviors without taking the agent's detailed behavior into account. Extending previous work in the Bayesian setting, we consider instead a worst-case setting in which the agent has a set of possible environments (MDPs) it could be in, and develop a commitment semantics that allows for probabilistic guarantees on the agent's behavior in any of the environments it could end up facing. Crucially, an agent receives observations (of reward and state transitions) that allow it to potentially eliminate possible environments and thus obtain higher utility by adapting its policy to the history of observations. We develop algorithms and provide theory and some preliminary empirical results showing that they ensure an agent meets its commitments with history-dependent policies while minimizing maximum regret over the possible environments.
\end{abstract}

\section{Introduction}
%ED: Edited first sentence, since it looks lazy to reuse the start of the abstract
When planning jointly, agents can benefit from making commitments to each other about what they will (or won't) do that affects another agent, so that other agents can form their own plans accordingly.
%In cooperative multiagent settings, it can often be beneficial for an agent to make commitments about its behavior to other agents. 
In the ideal case, commitments by an agent could allow the other agents to plan their behavior completely independently by relying on the commitments. For example, an agent could commit to free up a tool for another agent to use by a certain time and assuming that the only interaction among the two agents is the use of the tool, this can allow the other agent to plan independently. 

%ED: The paragraph below seems okay ... no obvious (easy) last-minute improvements
Some existing computational models of commitments characterize them using formal logic~\cite{CohenLevesque,Castelfranchi1995,Singh1999,Mallya2003,Chesani2013,Al-Saqqar2014}. When there is uncertainty about the consequences of actions, logical formulations associate conventions and protocols for managing such uncertainty~\cite{Jennings,xing2001formalization,Winikoff2006}.
%ED: Something seems wrong with the Jennings citation ... no year?
An alternative means of handling uncertainty, as in this paper, is to formalize commitments in decision-theoretic settings and explicitly allow for probabilistic guarantees of outcomes~\cite{xuan1999incorporating,Bannazadeh2010,witwicki2009commitment}.

An interesting challenge in making and keeping commitments arises when the committing agent expects to learn information about its environment while executing its plan. What should a probabilistic commitment mean in such a setting? Recently Zhang et al.~\shortcite{zhang2016} provided an answer to this question in sequential decision problems where the committing agent interacts with an environment modeled as a controlled Markov process with a prior distribution over possible reward functions and has already made a probabilistic commitment to achieve a state at a certain time. The committing agent observes rewards while taking actions and thereby can refine its distribution over possible reward functions after each action. They formalize the meaning of a probabilistic commitment as requiring the agent to ``execute a policy from the initial state that
properly affects the committed state variables in expectation'' (where this expectation is over both stochastic transitions and the effect of stochastic reward observations to the agent's knowledge during plan execution).

Our main contributions in this paper are to extend the work of Zhang et al. to the worst-case non-Bayesian setting in which the agent knows that the sequential decision making task it is facing is from one of a set of Markov Decision Processes (MDPs), where both reward and transition dynamics could differ across MDPs, and nonetheless guarantees, at least, the same commitment probability in all MDPs.
We propose a family of policy construction methods for the committing agent that adopts maximum regret as the performance criterion.
We 
%theoretically 
prove that policies constructed by the proposed methods respect this commitment semantics, and through experimental results we find they significantly outperform some baseline policies, such as the greedy policy that picks the next action minimizing myopic regret. 

\section{Example Domain}
For illustrative purposes, we first present a two-state example, Twin-States, before we formalize the general problem.
The Twin-States domain consists of two states with known deterministic transition dynamics but uncertain reward, as shown in Figure \ref{fig:TwinStates}.
The start state is A and the agent has three actions in each of the two states.
Action $a_0$ moves the agent to the other state with no reward, while actions $a_1$ and $a_2$ keep it in the original state.
Action $a_1$'s reward is 2 in state A and 3 in state B.
Action $a_2$'s reward can be any element of the set $\{1, 3 ,5\}$ in state A, and can be any element of set $\{0, 2, 4\}$ in state B.
%ED: When saying reward can be any element of the set, does that mean each time we do it we might have a new random draw from the set?  I know the answer is "no" but might a reader be confused?
The agent commits to being in state A at the time horizon with probability one.

\begin{figure}[ht]
\centering
\includegraphics[scale=0.8]{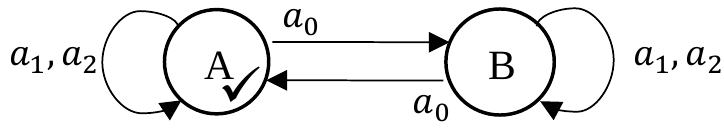}
\caption{Twin-States. See text for details.
%The agent starts in state A.
%Action $a_0$ moves the agent to the other state with no reward, while actions $a_1$ and $a_2$ keep it in the original state.
%Action $a_1$'s reward is 2 in state A and 3 in state B.
%Action $a_2$'s reward can be any element of the set $\{1, 3 ,5\}$ in state A, and can be any element of set $\{0, 2, 4\}$ in state B.
%The agent commits to being in state A at the time horizon with probability one.
}
%ED: It seems strange (and wasteful of space) to have the caption be redundant with the paragraph in the text right before this.
\label{fig:TwinStates} 
\end{figure}

Action $a_2$ could be more rewarding than action $a_1$, or less. 
If the horizon is large enough, the agent has enough time to figure out the reward of action $a_2$ in both states, and depending on the observed rewards it will thereafter have a clear preference for one state-action.
This behavior has two interesting properties.
%ED: I found the construction below hard to parse.  We have uncapitalized things that seem like they are continuing the sentence started right above this comment, but then they are broken up with full sentences?
First, it
chooses actions based on the 
%entire 
previous observations (i.e., history). Reward for actions that the agent has already taken helps it choose actions wisely in the future.
Second, it is constrained by the commitment. At the time step just before the time horizon in the Twin-States domain, if the agent is in state B it should take action $a_0$, otherwise it is in state A and should not take action $a_0$. In our experimental work below, we will show how our proposed methods can solve the Twin-States problem to compute such policies.

\section{Problem Formulation}
%\subsection{Non-Bayesian model uncertainty}
We consider settings in which an agent knows that its sequential decision making problem is one out of $K$ possible MDPs but does not know which MDP it is in at the start. 
%but does not have distributional information over the MDPs (i.e., we consider the non-Bayesian setting). 
We assume that all $K$ MDPs have the same state and action spaces but possibly different transition and reward functions.
During execution, the agent can observe the state and the reward and this can provide information about the MDP it is in.
The environment is formally defined by the tuple $\mathcal{E}=\langle \mathcal{S}, \mathcal{A}, \{P(k), R(k)\}_{k=1}^{K}, s_0\rangle$, 
where $\mathcal{S}$ and $\mathcal{A}$ are finite environment state (henceforth env-state) and action spaces, respectively,
and $s_0\in\mathcal{S}$ is the initial env-state. 
If the agent is in MDP $k$, then on taking action $a\in\mathcal{A}$ in env-state $s\in\mathcal{S}$ the agent receives reward $r=R_{sa}(k)$ and the environment transitions to env-state $s'$ with probability $P^{s'}_{sa}(k)$. Throughout we assume that the planning horizon is finite, which in turn implies that time is part of the env-state. Let $\mathcal{S}_t$ be the set of env-states at time step $t$, $S_t$ be a random variable indicating the env-state at time step $t$ whose specific realization is denoted $s_t$, $A_t$ be a random variable indicating the action taken at time step $t$, whose specific realization is denoted $a_t$, and let 
\begin{align} \label{eq:history}
h_t=\langle s_0,a_0,r_1,s_1,...,s_{t-1},a_{t-1},r_{t},s_{t} \rangle
\end{align}
be the history at time step $t$.
Because of the agent's lack of knowledge about the MDP it is facing, we consider history-dependent stochastic policies and use $\pi(a|h)$ to denote the probability of choosing action $a$ given history $h$ under policy $\pi$. 
During execution, history gives the agent knowledge about the true MDP it might be in or, equivalently, the MDPs that it cannot be in.
Formally, we can summarize the current history $h$ into a \emph{knowledge state}, $b:=\langle s, \kappa \rangle$, where $s$ is the current env-state, and $\kappa:=\{k: k\sim h\}$ is the set of indices of MDPs consistent with $h$.
Initially, the agent is in knowledge state $b_0=\langle s_0, \kappa_0 \rangle$ where $\kappa_0=\{1,2...K\}$.
Let $B_t$ be a random variable indicating the knowledge state at time step $t$, and $b_t$ be the knowledge state given history $h_t$. (In general there is a many to one mapping from histories to knowledge states.) We define the agent's planning objective below. 

\subsection{Commitment Semantics}
Note that there are two types of uncertainty in our setting. There is non-probabilistic uncertainty (i.e., incomplete knowledge) over which MDP the agent is facing, and there is probabilistic uncertainty (i.e., stochastic state transitions and possibly rewards) within an MDP. %Recall that in this paper we assume that a commit

Our commitment in this uncertain environment $\mathcal{E}$ is formally defined as follows.
\begin{definition}
	A probabilistic commitment $c$ is formally defined as a tuple $\langle \Phi, T, p \rangle$, where $\Phi \subset \mathcal{S}$ is the commitment  env-state space, $T$ is the commitment finite time horizon, 
%ED: changed phrase above from "finite commitment time horizon" to "commitment finite time horizon" - former sounded like the time horizon of a finite commitment, and I didn't know what a finite commitment is...
    and $p$ is the commitment probability. By making commitment $c$, the agent is constrained to follow a policy $\pi$, such that
\begin{align}\label{eq:semantics}
\Pr_{\pi} ( S_T \in \Phi | S_0=s_0; k ) \geq p, \forall k\in\kappa_0.
\end{align}
\end{definition}
From Equation \eqref{eq:semantics}, the semantics of a probabilistic commitment is clear:
the agent is constrained to follow a (in general history-dependent) policy, such that starting at the initial env-state, it will reach a env-state in the committed env-state space, $S_T$, at the time horizon, $T$, with at least the committed probability, $p$, no matter which MDP it is in.
Given probabilistic commitment $c$, let $\Pi_{c}$ be the set of all history-dependent stochastic policies that satisfy Equation \eqref{eq:semantics}.

\subsection{Minimax Regret}
In this paper, we are interested in finding a good policy given a probabilistic commitment, using maximum regret as the performance criterion.
Let
\begin{align*}
U^{\pi}(k)=\E_{\pi}\left[\sum_{t=0}^{T-1}R_{S_tA_t}(k)|S_0=s_0; k\right]
\end{align*}
be the expected cumulative reward under policy $\pi$ if the true MDP is $k$,
and let $U^*_{c}(k)$ be the expected cumulative reward under the optimal policy respecting the semantics of commitment $c$ if the true MDP is $k$:
\begin{align*}
U^*_{c}(k)=\max_{\pi\in\Pi_{c}} U^{\pi}(k).
\end{align*}
Finding $U^*_{c}(k)$ amounts to solving a standard constrained MDP problem and this can be done efficiently by linear programming~\cite{altman1999constrained}.
Given commitment $c$, let $\rho_{c}^{\pi}$ denote the maximum regret of policy $\pi$ under $c$, i.e.,
\begin{align*}
\rho_{c}^{\pi}=\max_{k\in\kappa_0}U^*_{c}(k)-U^{\pi}(k).
\end{align*}
Let $\Pi^*_c$ be the set of policies that minimizes the maximum regret while respecting the commitment semantics, 
\begin{align*}
\Pi^*_c = \{\pi : \pi\in\Pi_c, \rho_{c}^{\pi} = \min_{\pi'\in\Pi_{c}} \rho_{c}^{\pi'} \}.
\end{align*}
The agent's planning goal is to find a policy in $\Pi^*_c$.
We conclude this section with a series of formal observations showing that straightforward planning methods will not be enough to construct policies in $\Pi^*_c$.

%The reason is two fold: firstly one need to tell whether a history-dependent policy belongs to $\Pi_{c}$ or not, and secondly minimizing maximum regret itself, even without any constraint brought by the commitment, is hard.
Observation \ref{obs:MDPsBest} says that in general it is not sufficient to search over policies that are optimal for some MDP.
\begin{observation}
\label{obs:MDPsBest}
Let $\pi^*_c(k)=\argmax_{\pi\in\Pi_c} U^{\pi}(k)$ be a policy respecting commitment $c$ that is optimal if the true MDP is $k$.
Then, in general we have $\pi^*_c(k)\not\in\Pi^*_c, \forall k$.
\end{observation}
Observation \ref{obs:Greedy} says that in general it is not sufficient to greedily pick the next action that minimizes the maximum myopic regret.
\begin{observation}
\label{obs:Greedy}
Let $\pi_{G}$ be the greedy policy under which the agent selects the next action that minimizes the maximum myopic regret over the possible MDPs consistent with the current knowledge, i.e. 
\begin{align*}
a_t=\argmin_{a} \max_{k\in \kappa_t} \{ \max_{a'}R_{s_ta'}(k)-R_{s_ta}(k)\}.
\end{align*}
Then, in general we have $\pi_{G}\not\in \Pi^*_c$.
\end{observation}
%Observation \ref{obs:Stochastic} says it is possible that no policy in $\Pi^*_c$ is deterministic.
%\begin{observation}\label{obs:Stochastic}
%There exists an environment $\mathcal{E}$ such that no policy in $\Pi^*_c$ is deterministic.
%\end{observation}
Observation \ref{thm:Stochastic} says that it is possible that no policy in $\Pi^*_c$ is deterministic even if all MDPs in the environment are deterministic.
%, and therefore a regret minimizing agent would still prefer a commitment probability strictly less than one.
\begin{observation}\label{thm:Stochastic}
There exists an environment $\mathcal{E}$ where all MDPs are deterministic, i.e. $\forall k ,s, a ~\exists s' $ such that $P_{sa}^{s'}(k)=1$, and no policy in $\Pi^*_c$ is deterministic.
\end{observation}
The Twin-States domain provides a proof of the above observations by example as we verify in the Section on Empirical Results below.

Finally, we might think whenever the agent learns more about the true MDP during execution it is a good idea to re-plan from the current env-state with the original commitment probability.
Clearly, if during execution one can always find a policy that achieves the original commitment probability conditioned on the current env-state, such a re-planning approach will certainly respect the commitment semantics.
Observation \ref{obs:rhoUpdate} says that this is not always possible, and the example shown in Figure \ref{fig:ObsExample} verifies it.
\begin{figure}
\centering
\includegraphics[scale=0.7]{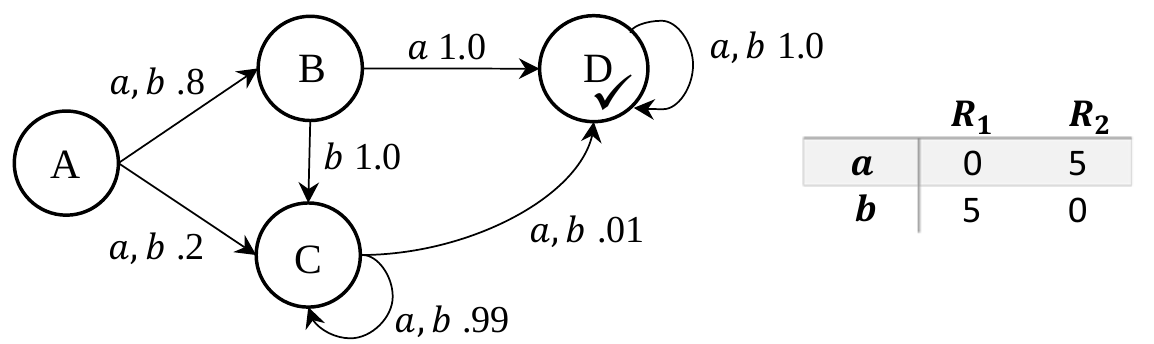}
\caption{Starting in state A, the agent commits to reaching the absorbing state D at time step two with at least probability $.8$. If the agent happens to be in state C at time step one, there is no plan that reaches state D from state C with probability at least $.8$ (verifying Observation \ref{obs:rhoUpdate}). There are two possible reward functions $R_1$ and $R_2$ shown above. Even though re-planning from state C does not yield a plan that leads to state D with probability $0.8$, the new plan will nonetheless reduce regret because at time step $1$ we will know which reward function applies and can therefore choose the more rewarding action in state C.}
\label{fig:ObsExample}
\end{figure}
\begin{observation}\label{obs:rhoUpdate}
There exists $\pi\in\Pi_c$ such that if the agent executes policy $\pi$ for the first $t > 0$ time steps starting in state $s_0$, the history generated $h_t$ is such that 
\begin{align*}
\forall \pi', \exists k\in \kappa_t~
\Pr_{\pi'} ( S_T \in \Phi | S_t=s_t; k ) < p.
\end{align*}
\end{observation}

\begin{figure} 
\begin{framed}
\begin{subequations}
\begin{align} 
&\min_{x}  ~\max_{k} { U^*_{c}(k)-U(k) } \label{program:CCNLObj}\\
&\mathrm{subject~to~~} \notag \\ 
& \forall k \quad U(k)=\sum_{s,a} x_{sa}(k) R_{sa}(k)  \label{program:CCNLUk} \\ 
&\forall k,s,a \quad x_{sa}(k)\geq 0 \label{program:CCNLxgeqzero} \\
&\forall k,s' ~ \sum_{a'}x_{s'a'}(k)=\sum_{s, a}x_{sa}(k)P_{sa}^{s'}(k)+\delta_{s's_0} \label{program:CCNLxdynamics}\\
&\forall k, k', s, a\quad\frac{x_{sa}(k)}{\sum_{a'}x_{sa'}(k)}=\frac{x_{sa}(k')}{\sum_{a'}x_{sa'}(k')} \label{program:CCNLxconsistency} \\
&\forall k\quad \sum_{s\in\Phi} \sum_{a} x_{sa}(k)\geq p \label{program:CCNLcommitment}
\end{align}
\end{subequations}
\end{framed}
\caption{CCNL program. It uses occupancy measures $x$ as decision variables. 
Constraint \eqref{program:CCNLUk} guarantees that $U(k)$ is the cumulative reward in MDP $k$, through which the maximum regret is expressed in objective function \eqref{program:CCNLObj}.
Constraints \eqref{program:CCNLxgeqzero} and \eqref{program:CCNLxdynamics} guarantee that $x(k)$ is a valid occupancy measure given that the initial state is $s_0$ and the transition function of the $k^{th}$ MDP is $P(k)$, where $\delta_{s's_0}$ is the Kronecker delta that returns 1 when $s'=s_0$ and 0 otherwise.
Constraint \eqref{program:CCNLxconsistency} guarantees that all $K$ occupancy measures have the same underlying Markov policy.
The commitment semantics is explicitly expressed in constraint \eqref{program:CCNLcommitment}.
The corresponding Markov policy can be recovered via Equation \eqref{policy:CCNL} in the main text.}
\label{fig:CCNLprogram}
\end{figure}

\section{Methods}
In this section we introduce several methods for constructing policies that respect the commitment semantics for a given commitment $c$.
\subsection{Commitment Constrained No-Lookahead}
Let $\Pi_{0}$ be the set of all Markov policies, i.e., policies that choose actions solely as a function of the current env-state (and ignore $\kappa$).
%Some policies do not depend the $\kappa$ knowledge but only on the current env-state to select the next action.
%We refer to them as no-lookahead polices and let $\Pi_{0}$ be the set of all no-lookahead polices.
Assuming $\Pi_0 \cap \Pi_c \neq \emptyset$, 
our Commitment Constrained No-Lookahead (CCNL) method of Figure~\ref{fig:CCNLprogram} finds a minimax regret Markov policy respecting the commitment semantics, which is a solution to the following problem:
\begin{align}\label{problem:CCNL}
\min_{\pi\in \Pi_c \cap\Pi_{0} } \rho_{c}^{\pi}.
\end{align}
%Given that the true MDP in the environment is 
For MDP $k$, each policy $\pi$ has a corresponding occupancy measure $x^{\pi}(k)$ for env-state-action pairs:
\begin{align*}
%\label{eq:DefOccupancyMeasure}
x_{sa}^{\pi}(k):=\E_{\pi}\left[\sum_{t=0}^{T-1}1_{\{S_t=s,A_t=a\}}|S_0=s_0; k \right].
\end{align*}
We will use shorthand notation $x(k)$ in place of $x^{\pi}(k)$ when policy $\pi$ is clear from the context.
If $\pi$ is a Markov policy, it can be recovered from its occupancy measure via
\begin{align}\label{policy:CCNL}
\pi(a|s)=\frac{x_{sa}(k)}{\sum_{a'}x_{sa'}(k)}.
\end{align}
%Then, problem \eqref{problem:CCNL} can be expressed by the mathematical program in 
Figure \ref{fig:CCNLprogram} presents our straightforward adaptation of the linear program for finding constrained-optimal policies in MDPs (see the caption of Figure~\ref{fig:CCNLprogram} for details).

\subsection{Commitment Constrained Lookahead}
During execution, the agent can observe the env-state transitions and reward and reason about the true MDP it might be in or, equivalently, the MDPs that it cannot be in. Thus, restricting the agent to Markov policies as in the previous section will lead to larger regret than is necessary. Here we consider the general case where the agent may choose actions based on the knowledge state (or equivalently history) for the first $0 < L \leq T$ steps, and use the env-state for the remaining time steps (if $L = 0$, we recover the Markov policy case above). We refer to $L$ as the knowledge-state-update boundary. The resulting {\bf \emph{$L$-updates}} policy has the form: 
%ED: I tried to bold L-updates above.  I missed it and was confused in the "$(L=)1$-updates" below.
%\begin{definition} \label{defn:Lstep}
%A history dependent stochastic policy $\pi$ is an \la if given the current history $h_t$, the policy chooses the next action according to
\begin{align*} 
\pi (a|h_t) =
\begin{cases} 
\pi (a|b_t) & t < L\\ 
\pi (a|s_t, b_L) & t \geq L,
\end{cases}
\end{align*}
where $b_t$ is the knowledge state consistent with $h_t$ and $b_L$ is the knowledge state consistent with $h_L$ when $t\geq L$. 
%\end{definition}
It is important to note that after the knowledge-state-update boundary, the policy conditions on both the env-state as well as the last updated knowledge state $b_L$. 

For example, Figure \ref{fig:onestep} shows a $(L=)1$-updates policy constructed in the Twin-States domain.
After taking some action in the initial knowledge state, depending on which knowledge state it actually ends up in at time $L=1$, it then executes a Markov policy, represented by a curve, all the way up to the horizon.
Those Markov policies starting from time step $L=1$ are not necessarily the same, which gives the agent flexibility of choosing different behaviors  based on the last-updated knowledge about the environment.
\begin{figure}
\centering
\includegraphics[scale=0.8]{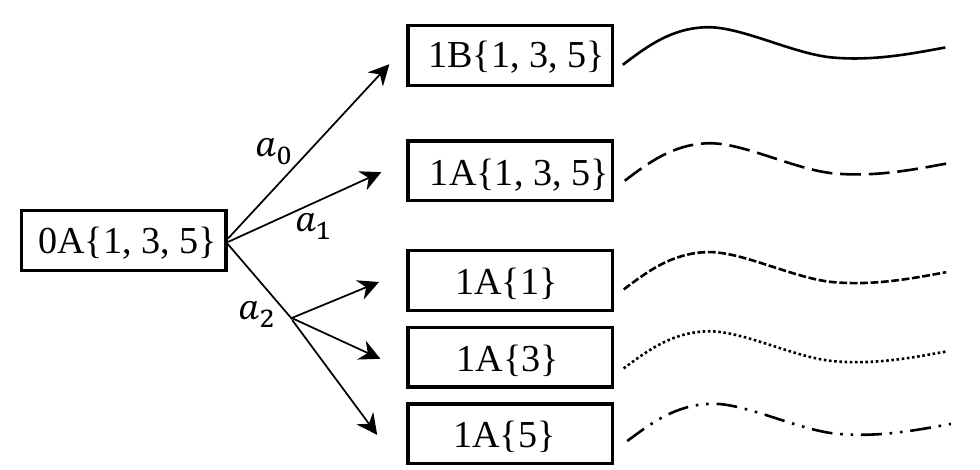}
\caption{Illustration of $1$-updates policy in the Twin-States domain.
Each box represents a reachable knowledge state with one step, where the number before letter ``A'' or ``B'' represents time and the number after the letter represents possible reward of action $a_2$ in state A.
Curves represent Markov policies the agent will follow after time step one.}
\label{fig:onestep} 
\end{figure}

Let $\Pi_{L}$ be the set of all $L$-updates policies.
Our Commitment Constrained Lookahead (CCL) method finds a minimax regret \la respecting the commitment semantics, which is a solution to the following problem:
\begin{align}\label{problem:CCL}
\min_{\pi\in \Pi_c \cap\Pi_{L} } \rho_{c}^{\pi}.
\end{align}
Problem \eqref{problem:CCL} can be expressed by the program in Figure \ref{fig:CCLprogram}.

\begin{figure*}[t!]
\begin{framed}
\begin{subequations}
\begin{align}
&\min_{x,y} ~\max_{k\in\kappa_0} \quad {U^*_{c}(k)   -U(k) } &&\text{(minimax regret objective)}\label{program:CCLObj}\\
&\mathrm{subject~to~}   \notag \\ 
&\forall k\in\kappa_0\notag \\
&~~U(k)= \sum_{b\in\mathcal{B}_{[0,L)}^{b_0}, a} y_{ba}(k) \tilde{R}_{ba}(k) +\sum_{b_L\in\mathcal{B}_L^{b_0}, s, a} x_{sa}^{b_L}(k)R_{sa}(k);   &&\text{(utility if MDP $k$ is true)}\label{program:CCLUk}\\
&\forall k, b, a\quad y_{ba}(k)\geq 0; \label{program:CCLygeqzero}\\
&\forall k , b'=\langle s',\kappa' \rangle \in \mathcal{B}_{[0,L]}^{b_0}\notag \\ 
&~~\sum_{a'}y_{b'a'}(k)=\sum_{b, a}y_{ba}(k)\tilde{P}_{ba}^{b'}(k)+\delta_{b'b_0}1_{\{k\in\kappa'\}}; &&\text{($\sum_{a}y_{b_0a}(k)=1$) }\label{program:CCLydynamics}\\
&\forall k, k', b\in\mathcal{B}_{[0,L)}^{b_0}, a \quad \frac{y_{ba}(k)}{\sum_{a'}y_{ba'}(k)}=\frac{y_{ba}(k')}{\sum_{a'}y_{ba'}(k')}; &&\text{(policies via $y(k)$ and $y(k')$ are consistent)}\label{program:CCLyconsistency}\\
&\forall k , b_L\in\mathcal{B}_L^{b_0}\quad y_{b_L}(k)=\sum_{a}y_{b_La}(k);&&\text{(define $y_{b_L}$ as the prob of reaching $b_L$)}\label{program:CCLybLk}\\
&\forall k , b_L\in\mathcal{B}_L^{b_0}, s, a\quad x_{sa}^{b_L}(k)\geq 0;  \label{program:CCLxgeqzero} \\
&\forall k , b_L=\langle s_L,\kappa_L \rangle\in\mathcal{B}_L^{b_0}, s' \notag \\
&~~\sum_{a'}x_{s'a'}^{b_L}(k)=\sum_{s, a}x_{sa}^{b_L}(k)P_{sa}^{s'}(k)+y_{b_L}(k)1_{\{k\in\kappa_L\}}\delta_{s's_L}; &&\text{($\sum_{a}x^{b_L}_{s_La}(k)=y_{b_L}(k)$) }\label{program:CCLxdynamics}\\
&\forall b_L\in\mathcal{B}_L^{b_0}, k, k', s, a \quad \frac{x_{sa}^{b_L}(k)}{\sum_{a'}x_{sa'}^{b_L}(k)}=\frac{x_{sa}^{b_L}(k')}{\sum_{a'}x_{sa'}^{b_L}(k')}; &&\text{(policies via $x(k)$ and $x(k')$ are consistent)}\label{program:CCLxconsistency}\\
&\forall k\in\kappa_0 \quad \sum_{b_L\in\mathcal{B}_L^{b_0}} \sum_{s\in\Phi,a}x_{sa}^{b_L}(k) \geq p &&\text{(commitment semantics)}\label{program:CCLcommitment}
\end{align}
\end{subequations}
\end{framed}
\caption{CCL program. To derive this program, we first define the knowledge state-based transition function $\tilde{P}_{ba}^{b'}(k)=1_{\{ k\in\kappa \}}\Pr(b'|b,a;k)$, and 
$\tilde{R}_{ba}(k)=1_{\{k\in\kappa\}}R_{sa}(k)$,
where $\tilde{P}_{ba}^{b'}(k)$ is the probability that the next knowledge state is $b'$ when taking action $a$ in knowledge state $b$, given that the true MDP is $k$.
Similarly $\tilde{R}_{ba}(k)$ is the reward of doing action $a$ in knowledge state $b$ given that the true MDP is $k$.
Note that if MDP $k$ is ruled out according to knowledge state $b$, then we define $\forall b', a, \tilde{P}_{ba}^{b'}(k)=0$, and $\forall a, \tilde{R}_{ba}(k)=0$.
Given policy $\pi$, one can use $\tilde{P}(k)$ to calculate the corresponding occupancy measure $y^{\pi}(k)$ for knowledge state-action pairs as follows:
%\begin{align*}
% %\label{eq:Defknowledge stateOccupancyMeasure}
$y^{\pi}_{ba}(k):=\E_{\pi}\left[\sum_{t=0}^{T-1}1_{\{B_t=b,A_t=a\}}|B_0=b_0; k \right].$
%\end{align*}
We use $\mathcal{B}_l^b$ to denote the set of reachable knowledge states after executing exactly $l$ actions from knowledge state $b$ and $\mathcal{B}_{[l_1,l_2]}^b=\bigcup_{l=l_1}^{l_2}\mathcal{B}_l^b$ to denote the set of reachable knowledge states from $b$ by executing any $l$ actions such that $l\in [l_1,l_2]$.
Because time is a state feature, $\mathcal{B}_l^b$ and $\mathcal{B}_{l'}^b$ are disjoint if $l\neq l'$.
CCL generates beforehand all reachable knowledge states from initial knowledge state $b_0$ within $L$ actions, $\mathcal{B}_{[0,L)}^{b_0}$. The state-action measures also enable us to express the expected cumulative reward conveniently in constraint \eqref{program:CCLUk} where the first RHS term sums up the reward of the first $L$ time steps and the second term the remaining $T-L$ time steps.
The state-action measures also enable us to express commitment semantics conveniently via constraint \eqref{program:CCLcommitment}.
Constraints \eqref{program:CCLygeqzero}, \eqref{program:CCLydynamics}, and \eqref{program:CCLyconsistency} on $y$ are the counterparts of \eqref{program:CCNLxgeqzero}, \eqref{program:CCNLxdynamics}, and \eqref{program:CCNLxconsistency} in Figure \ref{fig:CCNLprogram}.
Similarly, constraints \eqref{program:CCLxgeqzero}, \eqref{program:CCLxdynamics}, and \eqref{program:CCLxconsistency} are the counterparts for $x$.}
%ED: I don't think I've ever seen such a long caption!! ;-)
\label{fig:CCLprogram}
\end{figure*}

The program in Figure \ref{fig:CCLprogram} introduces as decision variables $y(k)$ and $x(k)$ for every possible MDP $k$, where $y(k)$ is the knowledge state-action occupancy measure if the true MDP is $k$, but only for those knowledge states reachable within the first $L$ time steps, and $x(k)$ is the env-state-action occupancy measure for the env-states in the remaining $T-L$ time steps if the true MDP is $k$. See the caption of Figure~\ref{fig:CCLprogram} for details.

Any \la $\pi_{L}$ respecting the commitment semantics can be derived from a feasible solution to the program in Figure \ref{fig:CCLprogram} via
\begin{align} 
\label{policy:CCL}
\pi_{L} (a|h_t) =
\begin{cases}
\pi_{L} (a|b_t) = \dfrac{y_{b _ta}(k)}{\sum_{a'}y_{b_ta'}(k)}  & t < L\\
\pi_{L} (a|s_t, b_L) =  \dfrac{x_{s_ta}^{b_L}(k) }{ \sum_{a'}x_{s_ta'}^{b_L}(k)}   & t \geq L.
\end{cases}
\end{align}
Theorem \ref{thm:CCL} states that CCL with knowledge-state-update boundary $L$ finds a minimax regret policy in $\Pi_c \cap \Pi_L$.
\begin{theorem}\label{thm:CCL}
If $\Pi_c \cap \Pi_L \neq \emptyset$ holds for commitment $c$, the program in Figure \ref{fig:CCLprogram} is feasible.
Let $x^*, y^*$ be its optimal solution, then the policy derived via Equation \eqref{policy:CCL} with $x^*, y^*$ is a minimax regret policy in $\Pi_c \cap \Pi_L$.
\end{theorem}
The proofs for Theorem~\ref{thm:CCL} and the theorems that follow are presented in the Appendix of a full paper available on arXiv.

Intuitively, a knowledge-state-update boundary greater than zero may help the agent choose actions according to its changing knowledge about the actual MDP it is in and therefore improve the performance.
Theorem \ref{thm:Lzero} says the maximum regret of the policy derived by CCL using any $L>0$ is upper bounded by the maximum regret of the policy derived by CCNL.
%ED: Does the paper tell the reader anywhere that the proofs of the theorems are in the supplementary material?  We can't assume the reviewer will even look to see if there is such material, so telling them (perhaps with the first theorem) that if they want to look that proofs of all the theorems is there makes sense to me...
\begin{theorem}\label{thm:Lzero}
If $\Pi_c \cap \Pi_0 \neq \emptyset$ holds for commitment $c$, the program in Figure \ref{fig:CCLprogram} is feasible for any $L\in[0, T]$.
Let $\pi_L^*$ be the policy derived by CCL using knowledge-state-update boundary $L$, then for any $L\in[0, T]$ we have
\begin{align*}
\rho_c^{\pi_L^*} \leq \rho_c^{\pi_0^*}.
\end{align*}
\end{theorem}

However, one has to be careful in using deeper boundaries because the performance of CCL is guaranteed to be monotonically non-decreasing in $L$ only when transition dynamics is invariant across MDPs, but this monotonicity cannot be guaranteed in general, as stated in Theorem \ref{thm:Lgeneral} and Theorem  \ref{thm:rewardonly}.
\begin{theorem}\label{thm:Lgeneral}
There exists an environment $\mathcal{E}$, a commitment $c$, $L'>L>0$ satisfying $\Pi_c \cap \Pi_L \neq \emptyset$ and $\Pi_c \cap \Pi_{L'} \neq \emptyset$, such that
\begin{align*}
\rho_c^{\pi_{L'}^*} > \rho_c^{\pi_{L}^*},
\end{align*}
where $\pi_L^*$ and $\pi_{L'}^*$ are the policies derived by CCL using boundaries $L$ and $L'$, respectively. 
\end{theorem}
\begin{theorem}\label{thm:rewardonly}
If the transition dynamics does not vary across MDPs in environment $\mathcal{E}$, i.e. $\forall k, k', P(k)=P(k')$, and $\Pi_c \cap \Pi_L \neq \emptyset$ for boundary $L$. 
Then for any $L'>L$ we have $\Pi_c \cap \Pi_{L'} \neq \emptyset$, and 
\begin{align*}
\rho_c^{\pi_{L'}^*} \leq \rho_c^{\pi_{L}^*},
\end{align*}
where $\pi_L^*$ and $\pi_{L'}^*$ are the policies derived by CCL using boundaries $L$ and $L'$, respectively. \end{theorem}

\subsection{Commitment Constrained Iterative Lookahead}
Commitment Constrained Iterative Lookahead (CCIL), as the name suggests, iteratively applies the CCL technique during execution.
Suppose starting from the initial knowledge state the agent executes the first $L$ actions prescribed by a minimax regret $L$-updates CCL policy $\pi^*_{L}$ derived by solving the program in Figure \ref{fig:CCLprogram} and ends up in knowledge state $b_L\in\mathcal{B}_L^{b_0}$.
Instead of executing the remaining $T-L$ actions prescribed by $\pi^*_{L}$, the agent can re-construct a new $L$-updates policy with an initial knowledge state now $b_L$.
This policy reconstruction is helpful because the agent gets more knowledge about the true MDP by observing the transitions and reward in the first $L$ steps.
Due to the changed initial knowledge state, naively sticking with the original commitment probability might lead to the difficulty stated in Observation \ref{obs:rhoUpdate}.
To respect the commitment semantics, the agent should instead plan with a commitment probability updated as follows.
%Here is a way to update the commitment probability in knowledge state $b_L$.
Let $b_L=\langle s_L,\kappa_L \rangle$, where $s_L$ is the current env-state, and $\kappa_L$ is the set of MDPs consistent with the history up to time step $L$.
For every possible MDP $k\in\kappa_L$, update the commitment probability as the \emph{achieved} probability if the agent were to stick with $\pi^*_{L}$ from $s_L$:
\begin{align}\label{eq:commitemtProbUpdate}
p(k) = \Pr_{\pi^*_{L}} ( S_T \in \Phi | S_L = s_L; k ).
\end{align}
Then, the agent can construct a new \la by solving the program in Figure \ref{fig:CCLprogram} with the following modifications:
\begin{enumerate}
\item Start from current knowledge state $b_L$ instead of $b_0$, i.e. replace every $b_0$ with $b_L$, and $\kappa_0$ with $\kappa_L$ in the program.
\item Plan with the updated commitment probabilities, i.e. replace $p$ in the last constraint of the program with $p(k)$ calculated as Equation \eqref{eq:commitemtProbUpdate}.
\item  Replace $U^*_{c}(k)$ with $U^*_{s_L,p(k)}(k)$ which is defined as the optimal objective value of the following problem:
\begin{align}
\label{eq:objectiveUpdate}
&\max_{\pi}\E_{\pi}\left[\sum_{t=L}^{T-1}R_{S_{t}A_{t}}(k)|S_L=s_L; k\right]\\
&\mathrm{subject~to~}  \Pr_{\pi} ( S_T \in \Phi | S_L = s_L; k )\geq p(k) \notag
\end{align}
which is the expected cumulative reward of the optimal policy that achieves commitment probability $p(k)$ from current env-state $s_L$ in MDP $k$.
\end{enumerate}
This modified program is guaranteed to be feasible because the original \la $\pi^*_{L}$ itself is a solution.
CCIL iteratively applies the above procedure every $L$ steps.
We outline CCIL in Algorithm \ref{algo:CCIL}, and Theorem \ref{thm:CCIL} formally states that it respects our commitment semantics.

\begin{theorem}\label{thm:CCIL}
If $\Pi_c \cap \Pi_L \neq \emptyset$ holds for commitment $c$ and boundary $L>0$, let $\pi^{\mathrm{IL}}_{L}$ be the history-dependent policy defined as Algorithm \ref{algo:CCIL}. We have $\pi^{\mathrm{IL}}_{L}\in\Pi_c$, i.e., CCIL respects the commitment semantics.
\end{theorem}

\begin{algorithm}
\caption{CCIL}
\label{algo:CCIL}
\SetAlgoLined
\SetKwInOut{Input}{Input}
\Input{Environment $\mathcal{E}=\langle \mathcal{S}, \mathcal{A}, s_0, \{P(k), R(k)\}_{k=1}^{K} \rangle$, commitment $c=\langle \Phi, T, p \rangle$, \\
integer $L\in(0, T]$ such that $\Pi_c \cap \Pi_L \neq \emptyset$; }
$b_0 \leftarrow \langle s_0, \kappa_0 \rangle$\;
$\pi_0 \leftarrow$ \la derived by solving the program in Figure \ref{fig:CCLprogram}\;
$t\leftarrow 0$\;

\While{$t< T$}
{
\For{$i=1,2,...,L$}
{
Take action $a_t\sim\pi_t(\cdot|b_t)$ and observe reward-next state transition $(s_t,a_t,r_t,s_{t+1})$\;
Update knowledge state as $b_{t+1} =\langle s_{t+1},\kappa_{t+1}\rangle$\;
$\pi_{t+1}\leftarrow\pi_t$\;
$t\leftarrow t+1$\;
}
\For{$k\in\kappa_t$}
{
$p(k) \leftarrow \Pr_{\pi_t} ( S_T \in \Phi | S_t = s_t; k )$\;
$U^*_{s_t,p(k)}(k)\leftarrow$ optimal objective value of \eqref{eq:objectiveUpdate}\;
}
$\pi_t\leftarrow$ policy derived by solving a modified version of the program in Figure \ref{fig:CCLprogram}: replacing every $b_0$ with $b_t$, $\kappa_0$ with $\kappa_t$, $p$ with $p(k)$, and $U^*_c(k)$ with $U^*_{s_t,p(k)}(k)$\;
}
\end{algorithm}

\subsection{MILP Formulation}
The CCL program in Figure \ref{fig:CCLprogram} introduces quadratic equality constraints \eqref{program:CCLyconsistency} and \eqref{program:CCLxconsistency} to ensure that the action selection rules derived from occupancy measures in all possible MDPs are identical.
These  constraints make the optimization problem non-convex and hard to solve.
%In principle, it's difficult to distinguish a local optimum from a global optimum in such problems.
In practice, many math-programming solvers are unable to handle programs with quadratic equality constraints.
Although some solvers can deal with such programs, they often need to take as input a feasible solution as the starting point, but finding a feasible solution by itself might be difficult, and the final solutions are usually sensitive to starting points.
Here we introduce a straightforward modification to the CCL program in Figure \ref{fig:CCLprogram} that replaces the quadratic equality constraints with mixed integer constraints, and therefore reformulates it into a Mixed Integer Linear Program (MILP) that has many available solvers.
The cost of this reformulation is that the derived policy is restricted to be deterministic.

Specifically, we introduce indicators $\Delta$ into the CCL program in Figure \ref{fig:CCLprogram} as additional decision variables with the following constraints:
\begin{align*}
& \forall b \in \mathcal{B}_{[0,L]}^{b_0}, a ~~\Delta_{ba}\in \{0,1\};~\text{(choose $a$ in $b$ iff $\Delta_{ba}=1$)}\\
&\forall b \in \mathcal{B}_{[0,L]}^{b_0} ~~ \sum_{a}\Delta_{ba}\leq 1;~ \text{(at most one action is chosen)}\\
& \forall k, b \in \mathcal{B}_{[0,L]}^{b_0}, a ~~ y_{ba}(k)\leq\Delta_{ba};~\text{($y$ is consistent with $\Delta$)}\\
& \forall b_L\in\mathcal{B}_L^{b_0}, s, a ~~\Delta^{b_L}_{sa}\in \{0,1\};~\text{(choose $a$ in $s$ iff $\Delta^{b_L}_{sa}=1$)}\\
&\forall b_L\in\mathcal{B}_L^{b_0}, s ~~\sum_{a}\Delta^{b_L}_{sa}\leq 1;~\text{(at most one action is chosen)}\\
& \forall k , b_L\in\mathcal{B}_L^{b_0}, s, a~~x^{b_L}_{sa}(k)\leq\Delta^{b_L}_{sa};~\text{($x$ is consistent with $\Delta$).}
\end{align*}
Then, any feasible solution with the above constraints replacing constraints \eqref{program:CCLyconsistency} and \eqref{program:CCLxconsistency} of the program in Figure \ref{fig:CCLprogram} yields a deterministic policy via Equation \eqref{policy:CCL}, which can be alternatively expressed using the indicator variables:
\begin{align}
\label{policy:CCLdeterministic}
\pi_{L} (a|h_t) =
\begin{cases} 
\pi_{L} (a|b_t) =1_{\{\Delta_{b_ta}=1\}}& t < L\\
\pi_{L} (a|s_t,b_L) =1_{\{\Delta^{b_L}_{s_ta}=1\}} & t \geq L.
\end{cases}
\end{align}
Note that the objective function of the program in Figure \ref{fig:CCLprogram} is non-linear due to the max operator.
However, it is easy to reformulate it into a linear objective function with a set of linear constraints.
In particular, one can introduce a scalar variable $z$ to replace the objective function \eqref{program:CCLObj} with
\begin{align*}
\min_{x,y} z
\end{align*}
and add the following constraints on $z$
\begin{align*}
\forall k\in\kappa_0 \quad z\geq  U^*_{c}(k)-U(k).
\end{align*}
With the above modifications, the program in Figure \ref{fig:CCLprogram} becomes a MILP.
The derived policy via \eqref{policy:CCLdeterministic} using an optimal solution to this MILP is a  deterministic policy that minimizes the maximum regret of all deterministic policies in $\Pi_c\cap\Pi_L$ (assuming this intersection is non-empty).

\section{Empirical Results}
We evaluate the performance of CCL and CCIL, under various choices of the boundary $L$, first on the Twin-States domain of Figure \ref{fig:TwinStates} that has uncertain rewards, and second on the Slippery T-Maze gridworld domain of Figure~\ref{fig:TMaze} that has uncertain transition dynamics.
CCL and CCIL MILP programs are solved using CPLEX 12.6.
\subsection{Results on the Twin-States Domain}
The main goals of the experiments on this domain are 
1) to provide a constructive proof of Observations \ref{obs:MDPsBest} to \ref{thm:Stochastic},
2) to evaluate the loss of the MILP formulation in a domain where an exact stochastic CCL policy can be computed, and 
3) to compare the performance of CCL and CCIL using various boundaries against simple policy construction methods.
%as benchmark in an environment where only reward is uncertain. 

\noindent\textbf{Short horizon.}
Here we set the time horizon to two so that we can find an exact stochastic minimax regret CCL policy\footnote{This exact policy is found not by solving the program in Figure \ref{fig:CCLprogram} but as follows.
Note that with only two actions available, the agent should not move to state B because it has to move back to state A using the second action and will get no reward at all.
We exploit this fact to compute an exact stochastic CCL policy, i.e. an exact solution to the program in Figure \ref{fig:CCLprogram} by solving another equivalent mathematical program:
1) Introduce $\pi(a|b),\pi(a|s;b_L)$ as decision variables, which are the probability of choosing action $a$ under an \la for $t<L$ and $t\geq L$, respectively. Because choosing $a_0$ is sub-optimal, we need to only consider $a\in\{a_1,a_2\}$.
2) Express the maximum regret as the objective function, the only constraint is that $\pi(\cdot|b),\pi(\cdot|s;b_L)$ should be valid probability measures.
The commitment semantics is automatically satisfied because we don't need to include action $a_0$.} and compare it with that found using the MILP formulation.

Figure \ref{fig:TwinStatesShortHorizon} plots the maximum regret under various choices of boundary $L$ using exact CCL, MILP-CCL, and MILP-CCIL.
Because exact CCL achieves better performance than MILP-CCL, it is clear that the derived policy must be stochastic, which provides a constructive proof of Observation \ref{thm:Stochastic}.

%\begin{proof}[Proof of Theorem \ref{thm:Stochastic}]
%Suppose in the Twin-States domain, the agent commits to being in state A at time horizon two. Because the horizon is two, any history-depend policy in the Twin-States domain can be expressed as a two-step lookahead policy. Further, because MILP-CCL with lookahead boundary $L$ with finds a deterministic policy that minimizes the maximum regret of all deterministic policies in $\Pi_c\cap\Pi_L$, MILP-CCL($L=2$) finds a deterministic policy that minimizes the maximum regret of all deterministic policies in $\Pi_c$. From Figure \ref{fig:TwinStatesShortHorizon} we know that the exact CCL($L=2$) achieves better performance than MILP-CCL($L=2$), and therefore any policy in $\Pi^*_c$ must be stochastic.
%\end{proof}

\begin{figure}[ht]
\centering
\includegraphics[scale=0.6]{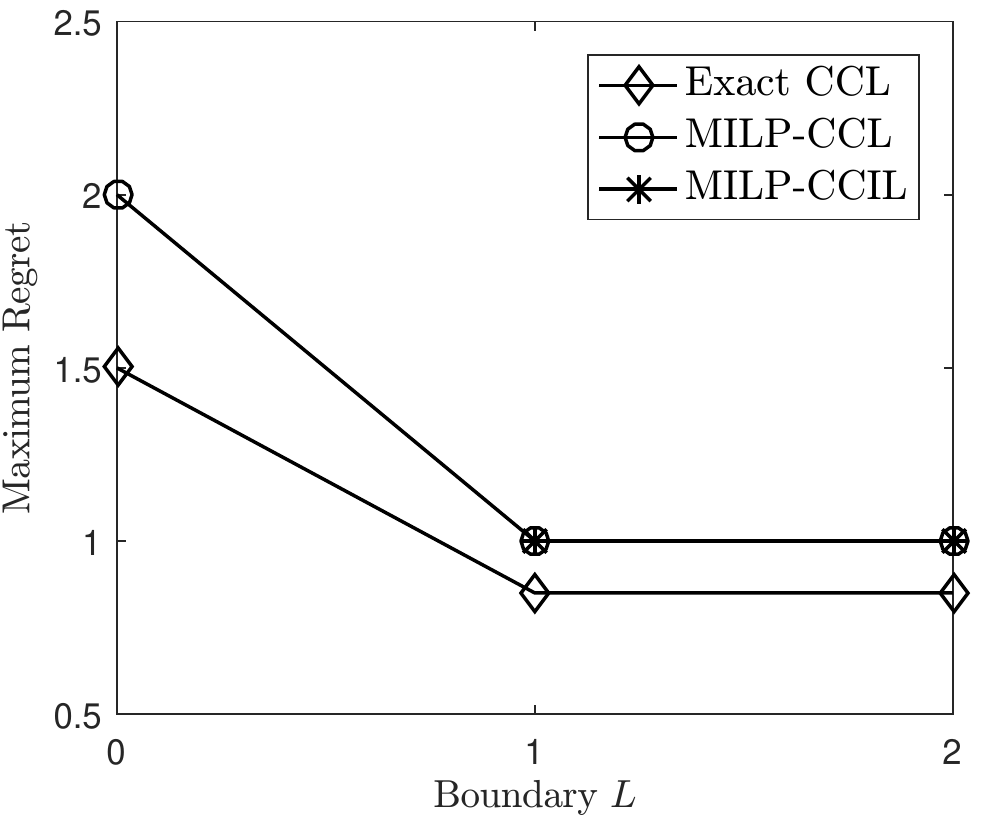}
\caption{Maximum regret in the Twin-States domain of exact CCL, MILP-CCL, and MILP-CCIL when horizon is two. Markers ``*'' for MILP-CCIL overlap with makers ``o'' for MILP-CCL when $L=1,2$. (Note that MILP-CCIL is not defined for $L = 0$).}
%ED: I still think it makes sense to briefly say that there is no datapoint for MILP-CCIL with L=0 because it is undefined for no lookahead?
\label{fig:TwinStatesShortHorizon} 
\end{figure}

\noindent\textbf{Longer horizon.}
Here we are concerned with comparing MILP-CCL and MILP-CCIL against the following baseline policy construction methods mentioned in Observation \ref{obs:MDPsBest} and Observation \ref{obs:Greedy} under longer than 2 time horizon.
\begin{itemize}
\item \textit{MDPs-Best}:
First find the optimal policies respecting the commitment semantics for every possible MDP, i.e.
$
\pi^*_k = \argmax_{\pi\in\Pi_c} U^{\pi}(k)
$.
The MDPs-Best policy is the one out of $\{\pi^*_k\}_{k=1}^{K}$ that minimizes the maximum regret.
\item \textit{Greedy}:
Select the next action that minimizes the maximum one-myopic regret over the possible MDPs consistent with the current history, i.e.
\begin{align*}
a_t=\argmin_{a\in\mathcal{A}_t} \max_{k\in\kappa_t} \{ \max_{a'\in\mathcal{A}_t}R_{s_ta'}(k)-R_{s_ta}(k)\}
\end{align*}
where $\mathcal{A}_t$ is the set of actions available at time $t$ that are chosen to guarantee the commitment semantics is respected.
For this domain, we let $\mathcal{A}_t=\{a_0, a_1, a_2\}$ if $t<T-1$.
When $t=T-1$, i.e. for the last action, $\mathcal{A}_t=\{a_0\}$ if $s_t$ is B, or $\mathcal{A}_t=\{a_1, a_2\}$ if $s_t$ is A.
\end{itemize}
Table \ref{table:TwinStatesMaxRegret} summarizes the results.
For MILP-CCL,
% the results show that 
performance is monotonic. 
It takes three steps to resolve the reward uncertainty by taking action $a_2$ in state A, moving to state B, and then taking action $a_2$ again.
This explains why $L$ larger than three does not improve the performance.
If the horizon is large enough, the agent should explore the reward of action $a_2$ in both states, then execute the action with the highest reward before going back to state A to respect the commitment semantics.
We find that is exactly what MILP-CCL($L\in[3,T]$) and MILP-CCIL($L=1$) do when horizon $T\geq 7$, which causes a max regret of 5 when reward of $a_2$ is the lowest (i.e., 1 in state A and 0 in state B).
\begin{table}[ht]
\centering
\caption{Max regret with varying horizon in Twin-States.}
\label{table:TwinStatesMaxRegret}
\begin{tabular}{c||c|c|c|c|c|c}
\hline
Horizon $T$&3  &  5 &  7 &  9 &  11 & 13 \\ \hhline{=||=|=|=|=|=|=}
Greedy     &5  &  5   &  9&  13   &  17  & 21  \\
MDPs-Best   &3  &  7  &  13&  19   &  25  &  31   \\
MILP-CCL, $L=0$      &3  &  6    &   10   &  15 & 19 & 22 \\
MILP-CCL, $L=1, 2$   &1  &  3  &  6&  8   &  9  &  11  \\ %\hline
MILP-CCL, $L\in[3,T]$  &1  &  3  &  5 &  5   &  5 &  5   \\%\hline
MILP-CCIL, $L=1$ &1  &  3  &  5 &  5   &  5 &  5  \\%\hline
\hline
\end{tabular}
\end{table}

\subsection{Results on the Slippery T-Maze}
The main goals of the experiments reported here are to evaluate CCL and CCIL with the MILP formulation in a domain where
1) the transition dynamics are uncertain, and
%ED: The above seems ungrammatical - the dynamics ARE uncertain ... but in another sense you want singular because it is the transition dynamics MODEL that is uncertain, right?
2) the commitment probability is less than one, and thus stochastic action selection is more likely to be crucial to achieving better performance.
\begin{figure}[ht]
\centering
\includegraphics[scale=0.5]{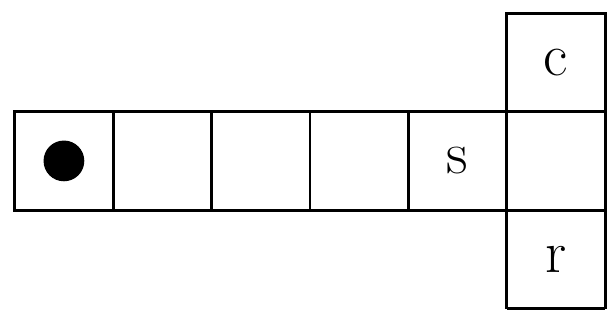}
\caption{Slippery T-Maze. See text for details.
%The agent starts in the cell with a black dot and can move in four directions. Staying in cell ``r'' results in a positive unit of reward every time step, but the agent commits to being in cell ``c'' at the time horizon. There are an uncertain number of consecutive slippery cells between cell ``s'' and the black dot cell.  In a slippery cell moving actions succeed with probability .8.Cell ``s'' is known to be slippery. 
}
\label{fig:TMaze} 
\end{figure}
The domain consists of two corridors that are connected as shown in Figure \ref{fig:TMaze}.
The agent starts in the cell with a black dot and can move in four directions.
Staying in cell ``r'' results in a positive unit of reward every time step, but the agent commits to being in cell ``c'' at the time horizon.
There are an uncertain number of consecutive slippery cells between cell ``s'' and the black dot cell. 
In a slippery cell movement actions succeed with probability .8.
Cell ``s'' is known to be slippery. 
The agent does not know in advance the number of slippery cells, which makes the transition dynamics uncertain.

Figure \ref{fig:TMazePlot} shows the results under commitment time horizon $T=10$ and commitment probability $p=0.6$.
The maximum regret of MILP-CCL is equal to the objective value of the mathematical program, which can be directly obtained, while the performance of MILP-CCIL is estimated by averaging many simulated episodes. The latter is seen to achieve better maximum regret than the former for low values of $L$.
Interestingly, and perhaps unexpectedly, unlike for the Twin-States domain, the performance of MILP-CCL is not monotonic in boundary $L$. The explanation lies in the fact that though the MILP-CCL policy is a deterministic function of history, the part of the policy that occurs after the boundary $L$ when viewed as a function of env-state alone is stochastic. This is because the knowledge-state at time $L$ is stochastic due to the stochastic transition dynamics (recall that the policy after $L$ is allowed to condition on the knowledge state at time $L$).
%ED: The previous sentence is a mouthful!  because of the fact ... due to the... and then a parenthetical!  I am guilty of such things too ;-)
Thus if $L$ is too large, the agent cannot take advantage of this stochasticity and suffers larger regret than for intermediate values of $L$. On the other hand if $L$ is too small, then the knowledge-state at $L$ is not informative enough to be helpful. Also interestingly, MILP-CCIL can take advantage of this implicit stochasticity using smaller $L$. However, when $L$ is large, MILP-CCIL achieves the same poor performance as MILP-CCL, because when $L$ is large the agent is likely to be in the vertical corridor where it no longer gets new knowledge about how many slippery cells there are and therefore iterative lookahead does not help.

\begin{figure}[t!]
\centering
\includegraphics[scale=0.58]{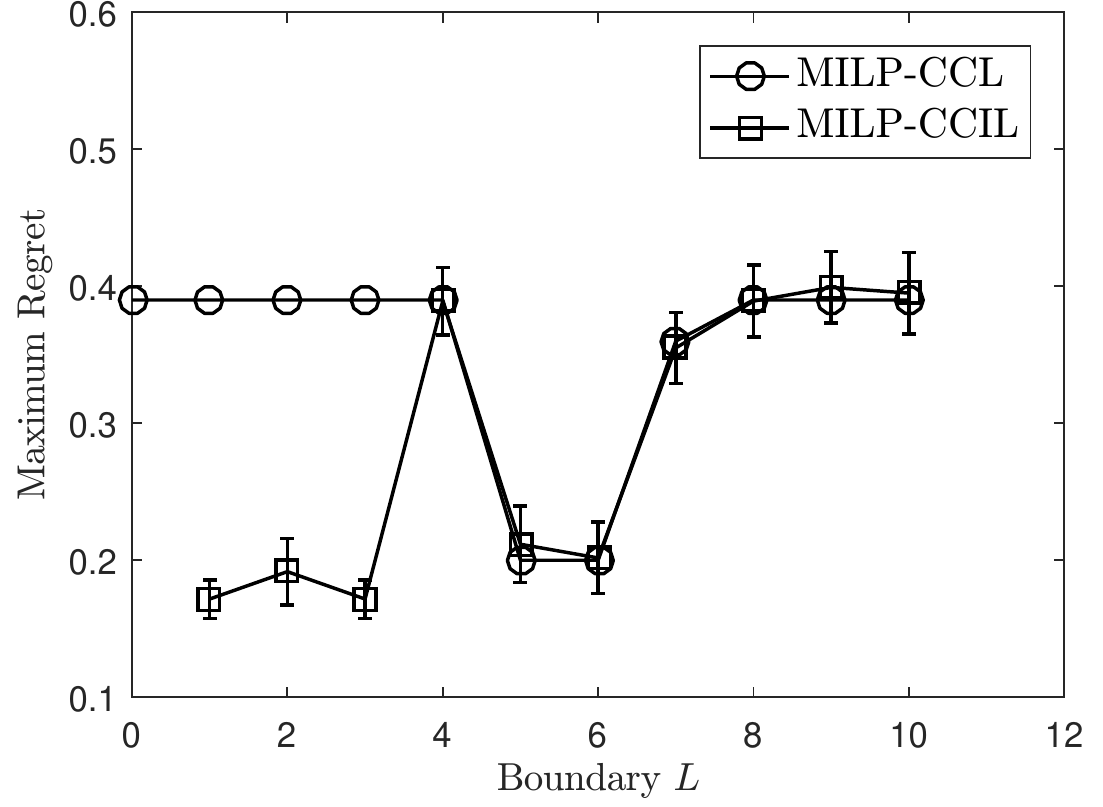}
\caption{Maximum regret in Slippery T-Maze of MILP-CCL and MILP-CCIL using various choices of the boundary under horizon $T=10$ and probability $p=0.6$.}
\label{fig:TMazePlot} 
\end{figure}
\section{Conclusion}
In this paper we developed a commitment semantics for achieving a specific state by a certain time with at least a certain probability in environments that have non-probabilistic uncertainty about the possible MDP the committing agent is facing as well as probabilistic uncertainty about the consequences of actions (within the true MDP). 
Our Commitment Constrained Lookahead (CCL) family of algorithms plan (offline) low-regret policies respecting the commitment semantics.
We provided analysis and empirical results on the impact of the knowledge-state-update boundary, which is an input-parameter to CCL, on the performance of the planned policy.
We extended CCL to Commitment Constrained Iterative Lookahead (CCIL), which is an iterative algorithm that adjusts the policy online. Exact CCL and CCIL require solving non-convex programs and thus we also introduced a MILP formulation that restricts the agent to deterministic policies. 
Our empirical results indicate that the MILP versions of both CCL and CCIL outperform baseline methods, and that CCIL is more robust than CCL.
\newpage
\noindent{\bf Acknowledgments} This work was supported in part by the Air Force Office of Scientific Research under grant FA9550-15-1-0039. Any opinions, findings, conclusions, or recommendations expressed here are those of the authors and do not necessarily  reflect the views of the sponsors.
\bibliographystyle{aaai}
\bibliography{icaps17}

\newpage

\section*{Technical Proofs}
\begin{proof} [Proof of Theorem \ref{thm:CCL}]
We need to show any policy in $\Pi_c \cap \Pi_L$ one-to-one maps to a feasible solution to the program in Figure \ref{fig:CCLprogram}.

For any policy $\pi\in\Pi_c\cap\Pi_L$, we are going to define vectors $m(\pi)$ and $n(\pi)$ such that they satisfy the constraints of the program in Figure \ref{fig:CCLprogram} if treated as the decision variables $x$ and $y$, respectively.
Given any policy $\pi\in\Pi_c\cap\Pi_L$, let $n(\pi,k)$ be its knowledge state-action occupancy measure if the true MDP is $k$ for knowledge states in $\mathcal{B}_{[0, L]}^{b_0}$, and $m(\pi,k)$ be its env-state-action occupancy measure for env-states from time step $L$ on.
\begin{align*}
\forall &b\in\mathcal{B}_{[0, L]}^{b_0}, a&\\
&n_{ba}(\pi,k) = \Pr_{\pi}(B_t=b,A_t=a|B_0=b_0;k)\\
\forall &s, a\\
&m^{b_L}_{sa}(\pi,k) \\
=&\begin{cases}
\Pr_{\pi}(S_t=s,A_t=a, B_L=b_L|B_0=b_0;k)  & t \geq L\\
0   & t < L
\end{cases}
\end{align*}
where $t$ is the time of knowledge state $b$.
We next show $n_{ba}(\pi,k)$ satisfies the constraints if treated as $y_{ba}(k)$ and $m^{b_L}_{ba}(\pi,k)$ satisfies the constraints if treated as $x^{b_L}_{ba}(k)$.

If treated as $y_{ba}(k)$, $n_{ba}(\pi,k)$ satisfies constraint \eqref{program:CCLygeqzero} because $n_{ba}(\pi,k)\geq 0$.

If treated as $y_{ba}(k)$, $n_{ba}(\pi,k)$ satisfies constraint \eqref{program:CCLydynamics} because if $b'=b_0$, for the LHS of \eqref{program:CCLydynamics} we have 
\begin{align*}
&\sum_{a'}n_{b'a'}(\pi,k)\\
=&\sum_{a'}n_{b_0a'}(\pi,k) \tag*{(because $b'=b_0$)}\\
=&\sum_{a'}\Pr_{\pi}(B_0=b_0,A_0=a'|B_0=b_0;k)\\
=&\Pr_{\pi}(B_0=b_0|B_0=b_0;k)\\
=&1
\end{align*}
and for the RHS of \eqref{program:CCLydynamics} we have
\begin{align*}
&\text{RHS of \eqref{program:CCLydynamics} }\\
=&\sum_{b, a}n_{ba}(\pi,k)\tilde{P}_{ba}^{b'}(k)+\delta_{b'b_0}1_{\{k\in\kappa'\}}\\
=&\sum_{b, a}n_{ba}(\pi,k)\tilde{P}_{ba}^{b_0}(k)+\delta_{b_0b_0}1_{\{k\in\kappa_0\}}\tag*{(because $b'=b_0$)}\\
=&0 + 1 =1 = \text{ LHS of \eqref{program:CCLydynamics} }
\end{align*}
If $b'\in \mathcal{B}^{b_0}_{[0,L]} \setminus \{b_0\}$, for the LHS of \eqref{program:CCLydynamics} we have
\begin{align*}
&\sum_{a'}n_{b'a'}(\pi,k)\\
=&\sum_{a'}\Pr_{\pi}(B_{t'}=b',A_{t'}=a'|B_0=b_0;k) \tag*{($t'$ is time of $b'$)}\\
=&\Pr_{\pi}(B_{t'}=b'|B_0=b_0;k)
\end{align*}
and for the RHS of \eqref{program:CCLydynamics} we have
\begin{align*}
&\text{RHS of \eqref{program:CCLydynamics} }\\
=&\sum_{b, a}n_{ba}(\pi,k)\tilde{P}_{ba}^{b'}(k)+\delta_{b'b_0}1_{\{k\in\kappa'\}}\\
=&\sum_{b, a}n_{ba}(\pi,k)\tilde{P}_{ba}^{b'}(k) \tag*{(because $b'\neq b_0$)}\\
=&\sum_{b, a}\Pr_{\pi}(B_{t}=b,A_{t}=a|B_0=b_0;k) \Pr(B_{t'}=b'|b,a;k) \tag*{($t$ is time of $b$, $t'$ is time of $b'$)}\\
=&\Pr_{\pi}(B_{t'}=b'|B_0=b_0;k)=\text{ LHS of \eqref{program:CCLydynamics} }
\end{align*}

If treated as $y_{ba}(k)$, $n_{ba}(\pi,k)$ satisfies constraint \eqref{program:CCLydynamics} because
\begin{align*}
\forall b\in\mathcal{B}_{[0, L)}^{b_0}, a\\
n_{ba}(\pi,k)=& \Pr_{\pi}(B_t=b,A_t=a|B_0=b_0;k)\\
=&\pi(a|b)\Pr_{\pi}(B_t=b|B_0=b_0;k)
\end{align*}
and therefore for any $k$
\begin{align*}
&\frac{y_{ba}(k)}{\sum_{a'}y_{ba'}(k)}\\
=&\frac{\pi(a|b)\Pr_{\pi}(B_t=b|B_0=b_0;k)}{\sum_{a'}\Pr_{\pi}(B_t=b,A_t=a'|B_0=b_0;k)}\\
=&\pi(a|b)
\end{align*}
which is independent of $k$.

Similarly one can show that $m^{b_L}_{ba}(\pi,k)$ satisfies constraints \eqref{program:CCLxgeqzero}, \eqref{program:CCLxdynamics}, and \eqref{program:CCLxconsistency} if treated as $x^{b_L}_{ba}(k)$.

If treated as $x^{b_L}_{ba}(k)$, $m^{b_L}_{ba}(\pi,k)$ also satisfies constraint \eqref{program:CCLcommitment} because $\pi\in\Pi_c$.

Constraints \eqref{program:CCLUk} and \eqref{program:CCLybLk} are naturally satisfied because they are defining new variables.

Now for the other direction, given a feasible solution $x, y$ to the program, let policy $\pi$ be the derived policy via \eqref{policy:CCL}.
Then $\pi$ is in $\Pi_L$ by definition.
Further we have $m^{b_L}_{sa}(\pi,k)=x^{b_L}_{sa}(k), n_{ba}(\pi,k)=y_{ba}(k)$, and therefore $\pi$ is also in $\Pi_c$ because $x$ satisfies commitment constraint \eqref{program:CCLcommitment}.
\end{proof}

\begin{proof} [Proof of Theorem \ref{thm:Lzero}]
Because CCL with boundary $L$ finds a minimax regret policy in $\Pi_c\cap\Pi_L$, it is sufficient to show
\begin{align*}
\forall L>0, \Pi_0 \subset \Pi_L
\end{align*}
This holds because given any Markov policy $\pi_0\in\Pi_0$, we can define an \la $\pi_L\in\Pi_L$ that is equivalent to $\pi_0$
\begin{align*}
\pi_L (a|h_t) =
\begin{cases} 
\pi_L (a|b_t)=\pi_0 (a|s_t) & t < L\\ 
\pi_L (a|s_t, b_L)=\pi_0 (a|s_t)& t \geq L
\end{cases}
\end{align*}
Thus, we know that $\pi_0\in\Pi_L$.
\end{proof}

We first prove Theorem \ref{thm:rewardonly} and then Theorem \ref{thm:Lgeneral}.
\begin{proof}[Proof of Theorem \ref{thm:rewardonly}]
It's sufficient to show that the statement holds when $L'=L+1$.
We next show that given any policy $\pi_L \in \Pi_{L}$, there exists an $(L+1)$-updates policy, $\pi_{L+1}$, that mimics $\pi_L$ when $P(k)=P(k') ~\forall k, k'$, and therefore $\rho_c^{\pi_{L+1}^*} \leq \rho_c^{\pi_{L}^*}$.

For the first $L$ actions, an $(L+1)$-updates policy can map the current knowledge state to a distribution of the next actions identical to $\pi_{L}$.
The action that is going to take at time step $L$ by $\pi_{L}$ can also be recovered by an $(L+1)$-updates policy, which gives
\begin{align*} 
\pi_{L+1} (a|h_t) =
\begin{cases} 
\pi_{L+1} (a|b_t) =\pi_{L} (a|b_t) & t < L\\ 
\pi_{L+1} (a|b_L) =  \pi_{L} (a|s_L, b_L) & t=L
\end{cases}
\end{align*}
If $P(k)=P(k') ~\forall k, k'$, then $\pi_{L+1}$ can also recover $\pi_{L}$ for $t\geq L+1$.
To see this, note if $P(k)=P(k') ~\forall k, k'$ we have
\begin{align}
\label{eq:bLbLplusone}
\Pr_{\pi_L}(b_L|b_{L+1};k)=\Pr_{\pi_L}(b_L|b_{L+1};k'), \forall \pi_L, k, k'\in \kappa_{L+1}
\end{align}
Therefore, under any \la $\pi_L$ and conditioned on being in knowledge state $b_{L+1}$ at time step $L+1$, the agent thereafter selects actions according to $\pi_L(\cdot|s_t,b_L)$ with probability defined in Equation \eqref{eq:bLbLplusone} that does not depend on $k$.
Because the transition function is known, the occupancy measure of $\pi_L$ for $t\geq L+1$ conditioned on any $b_{L+1}$ can be achieved by a stochastic Markov policy from $b_{L+1}$, which can be expressed by an $(L+1)$-updates policy as $\pi_{L+1}(\cdot|s_t,b_{L+1})$.
\end{proof}

\begin{proof}[Proof of Theorem \ref{thm:Lgeneral}]
In general $P(k)=P(k') ~\forall k, k'$ does not hold, and therefore condition \eqref{eq:bLbLplusone} in the proof of Theorem \ref{thm:rewardonly} does not hold.
Let $b_{L+1}=\langle s_{L+1}, \kappa_{L+1}\rangle$ be a knowledge state at time step $L+1$.

% For a \la $\pi_L$, conditioned on being in knowledge state $b_{L+1}$ at time step $L+1$, if the true MDP is $k\in\kappa_{L+1}$, the agent thereafter selects actions according to $\pi_L(\cdot|s_t,b_L)$ with probability $\Pr_{\pi_L}(b_L|b_{L+1};k)$.
% Therefore for a \la the state-action occupancy measure from $b_{L+1}$ on in general depends on $k$.
% In contrast following any $(L+1)$-updates policy $\pi_{L+1}$, conditioned on being in knowledge state $b_{L+1}$ at time step $L+1$, the agent thereafter selects actions according policy $\pi_{L+1}(\cdot|s_t,b_{L+1})$ which does not depend on $k$, and therefore any $(L+1)$-updates policy cannot mimic $\pi_L$.

Inspired by this we now give an example as a formal constructive proof.
In this example the env-state space is $\{0,1,2,3\}$ and let state 0 be the initial state.
There are two actions $a_0, a_1$ and $K=2$ possible MDPs.
The time horizon is three and the commitment probability is zero.
The transition dynamics and reward of these MDPs are as follows.
\begin{enumerate}
\item 
In MDP $k=1$, $\Pr(1|0,a;k=1)=0.9$, $\Pr(2|0,a;k=1)=0.1$, $\Pr(3|1,a;k=1)=\Pr(3|2,a;k=1)=1.0$ for both $a=a_0, a_1$.
State 3 is an absorbing state.
Doing $a_0$ in state 3 gives a positive unit of reward.
There's no reward elsewhere.
\item 
In MDP $k=2$, $\Pr(1|0,a;k=1)=0.1$, $\Pr(2|0,a;k=1)=0.9$, $\Pr(3|1,a;k=1)=\Pr(3|2,a;k=1)=1.0$ for both $a=a_0, a_1$.
State 3 is an absorbing state.
Doing $a_1$ in state 3 gives a positive unit of reward.
There's no reward elsewhere.
\end{enumerate}
The maximum regret when $L=2$ is 0.5, but $L=1$ can achieve a maximum regret of 0.1.
\end{proof}

\begin{proof}[Proof of Theorem \ref{thm:CCIL}]
We need to show
\begin{align*}
\Pr_{\pi^{\mathrm{IL}}_{L}}(S_T\in\Phi | S_0=s_0; k ) \geq p \quad \forall k\in\kappa_0.
\end{align*}
Let $\pi_{L}$ be the CCL $L$-updates policy derived from the program in Figure \ref{fig:CCLprogram}.
For any $k\in\kappa_0$, we can calculate the achieved commitment probability of $\pi^{\mathrm{IL}}_{L}$ by conditioning on the knowledge state it will visit at time $L>0$,
\begin{align*}
&\Pr_{\pi^{\mathrm{IL}}_{L}}(S_T\in\Phi | S_0=s_0; k )\\
=&\sum_{b_L}\Pr_{\pi^{\mathrm{IL}}_{L}}(B_L=b_L | S_0=s_0; k )\Pr_{\pi^{\mathrm{IL}}_{L}}(S_T\in\Phi | B_L=b_L; k )\\
=&\sum_{b_L}\Pr_{\pi_{L}}(B_L=b_L | S_0=s_0; k )\Pr_{\pi^{\mathrm{IL}}_{L}}(S_T\in\Phi | B_L=b_L; k )\\
\geq &\sum_{b_L}\Pr_{\pi_{L}}(B_L=b_L | S_0=s_0; k )\Pr_{\pi_{L}}(S_T\in\Phi | B_L=b_L; k )\\
=&\Pr_{\pi_L}(S_T\in\Phi | S_0=s_0; k )
\geq p
\end{align*}
The second equality holds because $\pi^{\mathrm{IL}}_{L}$ is identical to $\pi_{L}$ in the first $L$ steps.
The first inequality holds because CCIL iteratively applies $L$-step lookahead in Algorithm \ref{algo:CCIL} line 15 with the commitment probability achieved by the policy of the previous iteration calculated in line 12.
The modified program in Algorithm \ref{algo:CCIL} line 15 is always feasible because $\Pi_0 \subset \Pi_L$ which is shown in the proof of Theorem \ref{thm:Lzero}.
\end{proof}

\end{document}